\documentclass[letterpaper, 10 pt, conference]{ieeeconf}  
\IEEEoverridecommandlockouts                              
\usepackage[english]{babel}
\usepackage[T1]{fontenc}

\usepackage[nolist]{acronym}
\usepackage[cmex10]{amsmath}
\usepackage[font={small}]{caption}

\usepackage{algorithmic} 
\usepackage{algorithm}
\usepackage{subcaption}
\usepackage{amssymb}
\usepackage{bm}
\usepackage{booktabs}
\usepackage{color}
\usepackage{float}
\usepackage{graphicx}
\usepackage{hyperref}
\usepackage{import}
\usepackage{mathtools}
\usepackage{multirow}
\usepackage{multicol}
\usepackage{flushend}
\usepackage{outline}
\usepackage{paralist}
\usepackage{siunitx}
\usepackage{stmaryrd}
\usepackage{systeme}
\usepackage{url}
\usepackage{tikz}
\usepackage{booktabs}
\usepackage{balance}
\usepackage[backend=biber,
            style=ieee,
            natbib=true,
            sorting=none,
            url=false,
            doi=false,
            isbn=false,
            citestyle=numeric-comp,
            mincitenames=1,
            maxcitenames=1,
            maxbibnames=5,
            block=none
           ]{biblatex}
\usetikzlibrary{tikzmark}
\usepackage{esvect}
\usepackage[normalem]{ulem}
\usepackage{bbm}
\usepackage{siunitx}
\usepackage{wrapfig,lipsum,booktabs}

\usepackage{caption}
\usepackage{subcaption}
\usepackage{graphicx}
\usepackage{subcaption}
\usepackage{adjustbox}
\usepackage{xcolor}
\definecolor{royalblue}{RGB}{65,105,225}
\usepackage{flushend}
\usepackage{booktabs}
\usepackage{xspace}
\usepackage{etoolbox}
\usepackage{csquotes}
\usepackage{dblfloatfix}
\usepackage{cuted}
\newtheorem{assumption}{Assumption}

\newtheorem{theorem}{Theorem}

\hypersetup{
    colorlinks=true,
    linkcolor=purple,    
    citecolor=brown,    
    urlcolor=royalblue      
}

\newcommand{\website}{\newcommand{\website}{https://failure-aware-rl.github.io}}

\usepackage{colortbl}
\definecolor{ourcolor}{HTML}{99e0eb}
\definecolor{ourblue}{HTML}{27a2c3}
\definecolor{tablecolor}{HTML}{ccf2f5} 
\definecolor{tablecolor2}{HTML}{ffcdb4}
\definecolor{citecolor}{HTML}{fe7b5b}
\definecolor{grey}{rgb}{0.9, 0.9, 0.9}

\definecolor{gred}{rgb}{0.859,0.267,0.216}
\definecolor{ggreen}{rgb}{0.059,0.616,0.345}
\definecolor{deepblue}{HTML}{27a2c3}
\definecolor{deepred}{HTML}{fe7b5b}

\newcommand\ours{FARL\xspace}
\newcommand{\failures}[1][]{%
    \ifstrequal{#1}{full}{%
        Intervention-requiring Failures
    }{%
        \ifstrequal{#1}{single}{%
            IR Failure
        }{%
            IR Failures
        }%
    }\xspace
}

\title{\LARGE \bf 
Failure-Aware RL: Reliable Offline-to-Online Reinforcement Learning with Self-Recovery for Real-World Manipulation
\vspace{-0.35em}
}

\author{
Huanyu~Li$^{1,2*}$,
Kun~Lei$^{1,2*}$,
Sheng~Zang$^{4}$,
Kaizhe~Hu$^{1,3}$, 
Yongyuan~Liang$^{6}$,
Bo~An$^{4}$,
Xiaoli~Li$^{5}$,
Huazhe~Xu$^{1,3}$%
\thanks{$^{*}$Equal contribution.
$^{1}$Shanghai Qi Zhi Institute,
$^{2}$Shanghai Jiao Tong University,
$^{3}$IIIS, Tsinghua University,
$^{4}$Nanyang Technological University,
$^{5}$A*STAR Institute for Infocomm Research,
$^{6}$University of Maryland.}%
}

\begin{document}

\maketitle
\thispagestyle{empty}
\pagestyle{empty}

\begin{strip}
\vspace{-23mm}
\begin{center}
    \includegraphics[width=\textwidth, trim={0 0mm 0cm 0},]{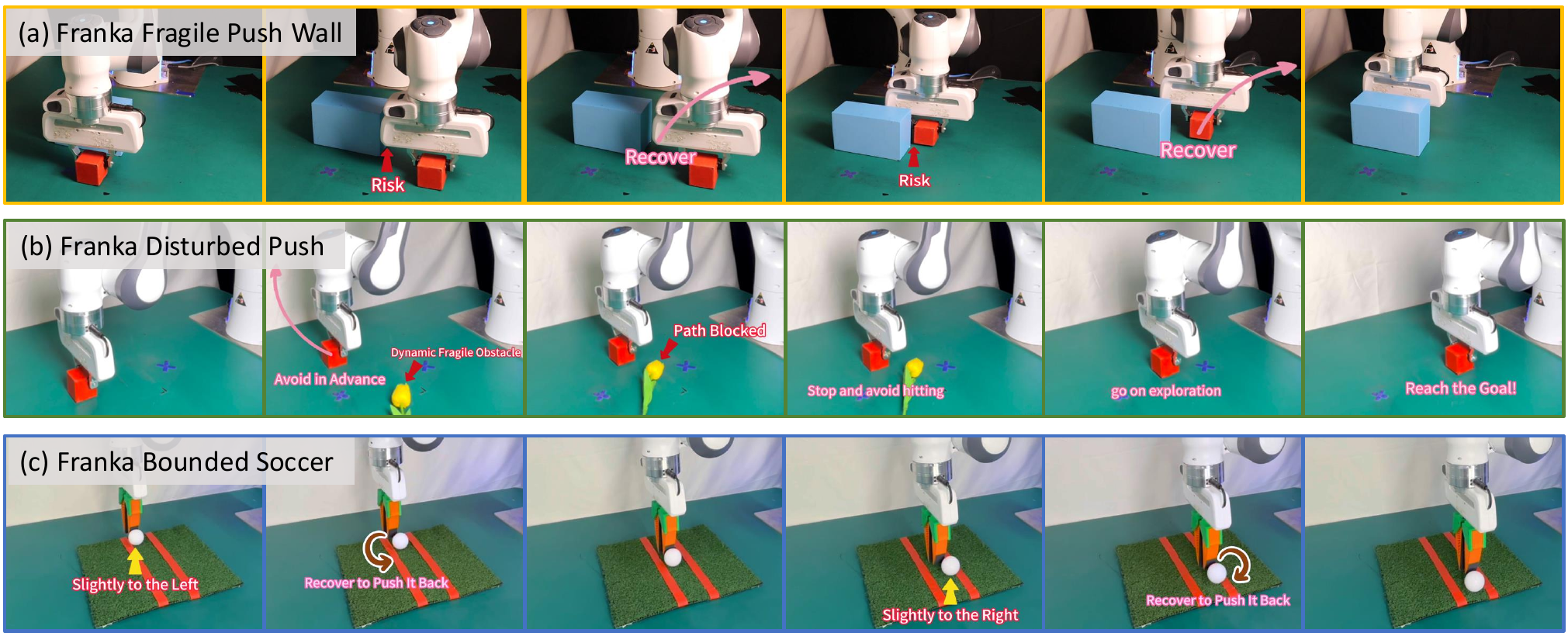}
    \captionof{figure}{
    \ours demonstrated on a Franka Emika Panda robot across three manipulation tasks: (a) pushing fragile objects while avoiding wall collisions, (b) pushing with dynamic obstacle avoidance, (c) soccer with boundary constraints (detailed in Section VI.C.1). \ours predicts potential failures and executes recovery actions, significantly reducing \failures[full] during real-world RL while improving task performance.}
    \label{fig:mainteaser}
\end{center}
\vspace{-4mm}
\end{strip}

\thispagestyle{empty}
\pagestyle{empty}

\begin{abstract}
Post-training algorithms based on deep reinforcement learning can push the limits of robotic models for specific objectives, such as generalizability, accuracy, and robustness.
However, \failures[full](\failures)~(e.g., a robot spilling water or breaking fragile glass) during real-world exploration happen inevitably, hindering the practical deployment of such a paradigm.
To tackle this, we introduce Failure-Aware Offline-to-Online Reinforcement Learning~(\ours), a new paradigm minimizing failures during real-world reinforcement learning.  We create FailureBench, a benchmark that incorporates common failure scenarios requiring human intervention, and propose an algorithm that integrates a world-model-based safety critic and a recovery policy trained offline to prevent failures during online exploration. 
Extensive simulation and real-world experiments demonstrate the effectiveness of \ours in significantly reducing \failures while improving performance and generalization during online reinforcement learning post-training. \ours reduces \failures by 73.1\% while elevating performance by 11.3\% on average during real-world RL post-training. Videos and code are available at \href{https://failure-aware-rl.github.io}{failure-aware-rl.github.io}.
\end{abstract}

\section{INTRODUCTION}

Learning efficient and accurate policies has been longed for years by roboticists and AI researchers.
However, pure imitation learning~\cite{wu2019behavior,shridhar2023perceiver,chi2023diffusion} often overfits the training distribution, while reinforcement learning is notoriously data-inefficient. 

Consequently, reinforcement learning (RL) based post-training becomes crucial for continuously refining specific objectives and adapting policies to the dynamic conditions of the real world. Specifically, offline-to-online RL~\cite{CQL, iql, lei2024unio}, where the agent is first trained offline on demonstrations and then fine-tuned through online RL, leverages the advantages of both demonstration-based pre-training and RL-based post-training.
However, a significant challenge in deploying RL in such settings is the occurrence of \failures[full](\failures) during the learning process. These failures stem from the intrinsic necessity for exploration in RL, which introduces randomness into the agent's actions. While exploration is crucial for learning optimal policies, it can result in actions that lead to irreversible damage or unsafe situations, such as breaking fragile objects, knocking items out of reach, or damaging the robot arm by hitting objects. Thus, \failures often necessitate human intervention to resolve, as they cannot be easily addressed through heuristic methods or reset-free RL techniques~\cite{eysenbach2017leavetracelearningreset, gupta2021resetfreereinforcementlearningmultitask, sharma2022autonomousreinforcementlearningformalism}.

To this end, we introduce the concept of \textbf{F}ailure-\textbf{A}ware Offline-to-Online \textbf{R}einforcement \textbf{L}earning~(\ours), where the agent refines its policy while minimizing \failures that would otherwise require human intervention.

Our method builds on an offline-to-online RL algorithm~\cite{lei2024unio}, which employs a policy gradient objective to unify online and offline RL in an on-policy manner. 

In this work, we focus on failure-aware post-training within real-world manipulation settings.

To study various failure scenarios in simulation environments, we introduce a new benchmark, \textbf{FailureBench}.

\textbf{FailureBench} builds upon existing simulation environments~\cite{yu2019meta} by incorporating common IR Failure scenarios that frequently occur in real-world manipulation tasks.
We simulate situations where failures necessitate human intervention, such as objects being pushed out of the workspace or the robot entering unsafe states. This benchmark enables us to evaluate how effectively the RL algorithms balance performance improvement and generalization while minimizing \failures during exploration.

Our research reveals that existing online and offline-to-online RL algorithms~\cite{lei2024unio, cal-ql, li2023proto} frequently encounter \failures in this setting due to the inherent exploration-exploitation trade-off in RL. To mitigate this problem, we introduce a safety critic based on a latent world model for failure prediction, along with a recovery policy designed to prevent failures foreseen by the safety critic.

Both components are trained offline using carefully curated demonstrations and later deployed to prevent failures during the online task policy post-training process.

Our framework demonstrates significant reductions in \failures across various simulated settings within FailureBench, while simultaneously enhancing task performance and generalizability. We validate the effectiveness of \ours through real-world experiments conducted on a Franka Emika Panda robot, showing that our methods substantially reduce the need for human intervention during training.
\textbf{Our contributions are summarized as follows}:

\begin{itemize}

    \item We identify a major barrier to deploying RL in real-world scenarios: \failures caused by exploration-induced randomness. To study such potential risks in existing methods, we introduce \textbf{FailureBench}, a benchmark that repurposes existing RL environments to study \failures[single] cases requiring human intervention, enabling evaluation of RL algorithms for both performance and failure minimization.
    \item We propose a failure-aware offline-to-online framework, including a specifically designed world model and recovery policy to minimize \failures while facilitating learning and adaptation with RL, theoretically justified by an ``advantage correction'' analysis to simultaneously enhance learning and safety.

    \item We conduct extensive experiments in both simulated environments and three challenging real robotic tasks susceptible to \failures to validate the effectiveness of our approach.
\end{itemize}

\section{Related Work}
Previous research in safe reinforcement learning has extensively examined safety through various approaches, focusing mainly on learning-from-scratch scenarios~\cite{ray2019benchmarking,p3o,cpo}. However, with advances in modern robotic models that learn from demonstrations~\cite{wu2019behavior, shridhar2023perceiver, chi2023diffusion, yuan2025hermes, he2024learning, jiang2024robots}, there is an increasing recognition of the importance of online post-training strategies. Consequently, this work shifts its focus to a failure-aware offline-to-online reinforcement learning setting.

\textbf{Safe RL.} Constrained Markov Decision Processes (CMDPs)~\cite{altman2021constrained, ji2024ace, xu2023drm} have gained significant attention within the RL community, particularly in the context of constrained and safe real-world learning. Safe and recovery RL seeks to address two primary challenges~\cite{liu2024safe}: the development of effective safety constraints~\cite{wachi2024survey} and the formulation of safe actions. To tackle these challenges, a substantial body of work utilizes the established online RL workflow to balance task rewards and constraints, their focus lies on optimization techniques, such as Lagrangian relaxation~\cite{liang2018accelerated, tansehoon, lg1}, Lyapunov functions~\cite{chow2018lyapunov, chow2019lyapunov} and robustness guarantees~\cite{liang2022efficient, liu2022robustness, liang2024gametheoretic, liu2024beyond}. However, a common issue is that these constraints are often enforced prematurely, which limits exploration and can reduce overall performance~\cite{he2024agile}. 
The term ``recovery'' in safe RL can be misleading. True recovery methods like damage adaptation~\cite{cully2015robots} address post-failure adaptation, while many ``recovery RL'' methods actually prevent failures. Recovery RL~\cite{thananjeyan2021recovery} and ABS~\cite{he2024agile} predict constraint violations and employ recovery policies to avoid unsafe states before failures occur. Similarly, recent prevention methods include Control Barrier Functions~\cite{cbf} that ensure forward invariance of safe sets, and predictive safety filters~\cite{wabersich2021predictive} that use MPC to modify unsafe control inputs. SafeDreamer~\cite{huang2024safedreamer} integrates Lagrangian methods into the DreamerV3 planning process for model-based safe RL. While SafeDreamer focuses on traditional safe RL tasks in simulation, our \ours addresses offline-to-online post-training of pre-trained policies in the real world. Alternative approaches include hierarchical safe RL~\cite{dalal2018safe, xiao2024safe}, which utilizes structural dynamics, and methods incorporating safety certificates from control theory~\cite{cheng2019end, 2020arXiv200504374N}.

\textbf{Offline-to-online RL.}
Offline RL aims to address distributional shift issues that arise when a policy encounters out-of-distribution (OOD) state-action pairs. Prior methods mitigate this challenge by incorporating conservatism~\cite{CQL} or constraint-based regularization~\cite{iql, onesteprl, bppo}, thereby either discouraging the policy from exploring OOD regions or providing conservative value estimates.
Once pre-trained with offline data, policies can be further improved through online fine-tuning. However, directly applying standard online off-policy RL algorithms often leads to severe performance degradation due to distributional shift during online exploration~\cite{yuan2023rlvigen}. 
Uni-O4~\cite{lei2024unio} directly applies the PPO~\cite{ppo} objective to unify offline and online learning, eliminating the need for extra regularization. RL-100 \citep{rl100} combines iterative offline RL with online RL to train diffusion-based policies for deployable robot learning.
However, deploying such offline-to-online methods in real-world robotic systems remains unsafe and expensive due to the high risk of \failures during online exploration.
In this work, we build upon Uni-O4 and extend it to address safety concerns in real-world manipulation tasks, enabling safer and more efficient policy refinement in real-world environments.

\begin{figure*}[t]
    \centering
    \includegraphics[width=0.91\linewidth, trim={0 0mm 0cm 0}, clip]{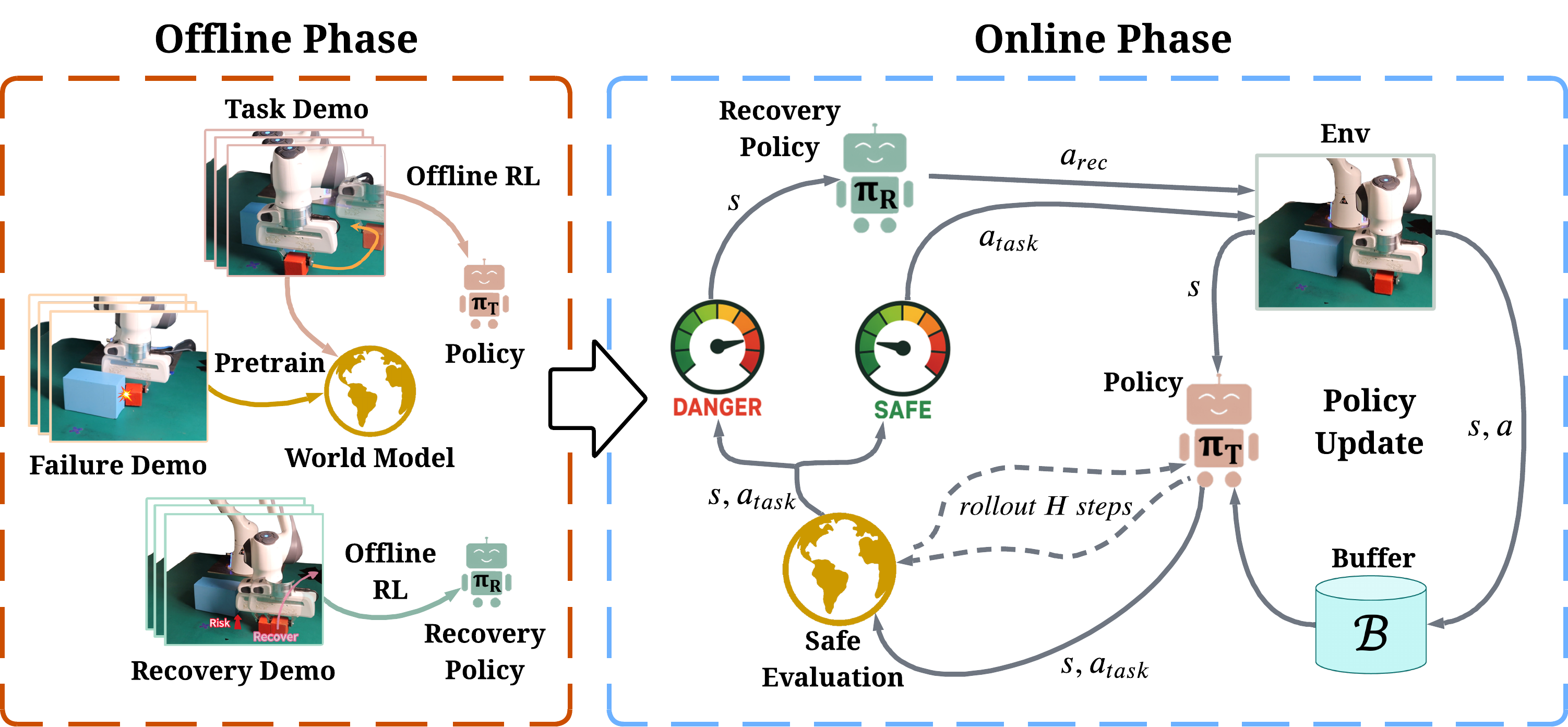}
    \caption{\small The entire training pipeline consists of two main phases: 1) the offline phase, which involves pre-training the task policy, recovery policy, and world model, and 2) the online phase, during which the task policy is fine-tuned within safe exploration settings.}
    \label{fig:teaser}
    \vspace{-0.5cm}
\end{figure*}

\section{Problem Statement}

We consider RL-based post-training under Constrained Markov Decision Processes (CMDPs)~\cite{altman2021constrained}. Following \citet{thananjeyan2021recovery}, we limit constraint costs to binary indicator functions that identify constraint-violating states. This can be described by tuple $\mathcal{M}\left(\mathcal{S}, \mathcal{A}, P(\cdot \mid \cdot, \cdot), R(\cdot, \cdot), \gamma, C(\cdot), \gamma_{\text {risk }}, \mu\right)$. Here, $\mathcal{S}$ and $\mathcal{A}$ are the state and action spaces. $P:\mathcal{S}\times\mathcal{A}\times\mathcal{S}\rightarrow{}[0,1]$ denotes the stochastic dynamics model which maps the given state and action to a probability distribution over the next state. $R:\mathcal{S}\times\mathcal{A}\rightarrow{}\mathbb{R}$ is the reward function, and $\gamma$ denotes the discount factor. $C:\mathcal{S}\times\mathcal{A}\rightarrow{}\{0,1\}$ is a constraint cost function, which denotes whether a state and action pair violate the designed constraint and it is associated with a threshold $\gamma_{risk}$.

Given a set of behavior policies $\pi \in \Pi$, the objective of the expected return can be defined as $R=\mathbb{E}_{\pi,\mu, P}\left[\sum_t\gamma^t \left(s_t, a_t\right)\right]$. 
We consider the $H$-steps discounted probability of constraint violation, which can be defined as 

\begin{equation}
\label{eq:c_pi_h}
\begin{aligned}
C_H^{\pi}&=\mathbb{E}_{\boldsymbol{\pi},\boldsymbol{\mu},\boldsymbol{P}}\left[\sum_{i=t}^{t+H}\boldsymbol{\gamma}_\mathrm{risk}^{i-t}C(s_{t}, a_t)\right] \\ &=\sum_{i=t}^{t+H}\boldsymbol{\gamma}_\mathrm{risk}^{i-t}\mathbb{P}\left(\boldsymbol{C}(s_{t},a_t)=1\right)
\end{aligned}
\end{equation}
in which the constraint costs are defined by the binary indicator functions. 

Thus, the objective of RL under CMDPs is:

\begin{equation}
\label{eq:cmdp}
\begin{aligned}
\pi^* = \underset{\pi \in \Pi}{\operatorname{argmax}} \, R^\pi \; \text{s.t.} \; C_H^{\pi} \leq \boldsymbol{\varepsilon}_{\mathrm{safe}}
\end{aligned}
\end{equation}
The constraints define a set of feasible policies, i.e., \{$\pi \in \Pi: \quad C_H^{\pi} \leq \boldsymbol{\varepsilon}_{\mathrm{safe}}$\}. We optimize this objective over the set of feasible policies.

In this work, we propose a failure-aware offline-to-online RL algorithm designed to optimize safe policy exploration and refinement. The core concept is inspired by recovery RL. \textbf{\textit{However, our focus is on post-training, which enhances robotic models learned from demonstrations by refining specific capabilities and providing a better policy initialization for safety, rather than relying on unnecessary random exploration from scratch.}}

\section{Method}
We formally introduce \ours, our offline-to-online failure-aware RL framework. Our overall algorithm pipeline comprises both offline and online phases, as summarized in Figure \ref{fig:teaser}: 
1) We first pre-train a task policy, a recovery policy, and a world model. The task policy is trained using task demonstrations that show successful task completion. The recovery policy is trained with recovery demonstrations that illustrate how to avoid or escape from near-failure states. The world model leverages both task demonstrations and failure demonstrations that capture state-action sequences leading to \failures.
2) Next, we fine-tune the task policy, which is optimized according to Eq. \eqref{eq:cmdp}, and guided by the recovery policy. This policy directs the agent back to the state-action pair $(s, a)$ where $C_H^{\pi} \leq \boldsymbol{\varepsilon}_{\mathrm{safe}}$. During the online phase, the recovery policy and world model remain fixed to minimize \failures during exploration.
\subsection{Offline Pre-training}

We pre-train a task policy to serve as the initial framework for online fine-tuning. In addition, we also pre-train a recovery policy and a world model, both designed to guide the task policy's exploration and help prevent \failures during the online phase. For the task policy pre-training, we follow the training pipeline established in Uni-O4~\cite{lei2024unio}. Initially, the policy undergoes behavior cloning, followed by fine-tuning using the objective:
\begin{equation}
\vspace{-1.5em}
\label{eq:ppo loss}
\begin{aligned}
    J_{k}\left(\pi\right) = & \mathbb{E}_{s \sim \rho_{\pi}\left(\cdot\right), a \sim \pi_{k}\left(\cdot|s\right)}\bigg[
    \min\Big(
    r(\pi)A(s,a), \\
    & \text{clip}\left(r(\pi),1-\epsilon,1+\epsilon\right)A(s,a)\Big)
    \bigg]
\end{aligned}
\vspace{1em}
\end{equation}
where $\rho_{\pi}$ is the stationary distribution of states under policy $\pi$, $r(\pi)=\frac{\pi\left(a|s\right)}{\pi_k\left(a|s\right)}$ denotes the importance sampling ratio between the target policy $\pi$ and behavior policy $\pi_k$, clip$(\cdot)$ is a conservatism operation that constrains the ratio, hyper-parameter $\epsilon$ is used to adjust the degree of conservatism, and $A(s_t, a_t)$ is the advantage function.
Subsequently, we continue fine-tuning within the environment to optimize the objective in Equation \ref{eq:ppo loss} using GAE advantage estimation. 
Next, we train the recovery policy through behavior cloning, utilizing recovery demonstrations. The training process during the offline phase mirrors that of the task policy, beginning with behavior cloning and then fine-tuning via Uni-O4 in an offline context. However, we avoid fine-tuning the recovery policy during the online phase due to the limited availability of failure data. We observed that this enhances the safe exploration of the task policy in real-world environments.

We also pre-train a world model with both task and failure demonstrations, specifically for predicting future failures. The world model focuses on a limited number of near-future steps rather than the entire episode in recovery reinforcement learning~\cite{thananjeyan2021recovery}. Our findings indicate that planning a short distance into the future can effectively minimize \failures in scenarios of real-world manipulation. To achieve this, we augment the world model for failure prediction by introducing a constraint prediction head. The training objective of our world model can be defined as:
\begin{align}
    \label{eq:wm}
    \mathcal{J}(\theta;\Gamma)=\sum_{i=t}^{t+H}\lambda^{i-t}\cdot\mathcal{L}(\theta;\Gamma_i),
\end{align}
which can be computed step-wise as follows:
\begin{align}
\label{eq:wm_each}
 \mathcal{L}(\theta;&\Gamma_i) =  \ c_1 
 \underbrace{\|R_\theta(\mathbf{z}_i,\mathbf{a}_i) - r_i\|_2^2}_{\mathrm{reward}} \\
& + c_2 \cdot \underbrace{\text{CE}(Q_\theta(\mathbf{z}_i,\mathbf{a}_i) - \left(r_i + \gamma Q_{\theta^-}(\mathbf{z}_{i+1}, \pi_\theta(\mathbf{z}_{i+1}))\right))}_{\mathrm{value}} \\
& + c_3 \cdot \underbrace{\text{CE}(d_\theta(\mathbf{z}_i,\mathbf{a}_i) - h_{\theta^-}(\mathbf{s}_{i+1}))}_{\text{latent state consistency}}\\
& + c_4 \underbrace{\|C_\theta(\mathbf{z}_i,\mathbf{a}_i) - c_i\|_2^2}_{\mathrm{constraint}}\\
& + c_5 \underbrace{\|S_\theta(\mathbf{z}_i) - s_i\|_2^2}_{\mathrm{decoder}},
\end{align}

where the model consists of seven components:
\vspace{-1.5em}

\begin{equation}
\vspace{-0.5em}
\footnotesize
\begin{aligned}
\text{Representation:} & \quad \mathbf{z}_{t} = h_{\theta}\left(\mathbf{s}_{t}\right) \\
\text{Latent dynamics:} & \quad \mathbf{z}_{t+1} = d_{\theta}\left(\mathbf{z}_{t}, \mathbf{a}_{t}\right) \\
\text{Reward:} & \quad \hat{r}_{t} = R_{\theta}\left(\mathbf{z}_{t}, \mathbf{a}_{t}\right) \\
\text{Value:} & \quad \hat{q}_{t} = Q_{\theta}\left(\mathbf{z}_{t}, \mathbf{a}_{t}\right) \\
\end{aligned}
\qquad
\begin{aligned}
\text{Policy:} & \quad \hat{\mathbf{a}}_{t} \sim \pi_{\theta}\left(\mathbf{z}_{t}\right) \\
\text{Decoder:} & \quad \hat{\mathbf{s}}_{t} \sim S_{\theta}\left(\mathbf{z}_{t}\right) \\
\text{Constraint:} & \quad \hat{\mathbf{c}}_{t} \sim C_{\theta}\left(\mathbf{z}_{t}, \mathbf{a}_{t}\right) \\
\end{aligned}
\end{equation}

where $s_t$ is the observation at time-step $t$, which is encoded to a representation $\mathbf{z}$ by an encoder $h_{\theta}$. Conditioned on the representation, each head of the dynamics model predicts the representation of the next state, the reconstructed state, the constraint, a single-step reward, a state-action $Q$-value, and an action that maximizes the $Q$-function.

We incorporate value and reward signals during training to enrich latent representations while providing gradient-based regularization that prevents overfitting to constraint patterns.

After pre-training, we infer the discounted near-future constraint of $\pi_{task}$ mentioned in Eq. \eqref{eq:c_pi_h} as:
\begin{equation}
\begin{aligned}
C_H^{\pi} = \mathbb{E}_{\boldsymbol{\pi}}\bigg[ & \sum_{i=t}^{t+H} \boldsymbol{\gamma}_\mathrm{risk}^{i-t}\big[ C_{\theta}(z_{i}, a_i) \,|\, z_i=h_{\theta}(s_i); \\
 & a_i=\pi_{task}(S_{\theta}(z_i)); z_{i+1}=d_{\theta}(z_i,a_i) \big] \bigg]
\end{aligned}
\end{equation}

\subsection{Online fine-tuning with recovery}

The task policy is fine-tuned by online PPO with the policy and value network initialization from offline pre-training, following Uni-O4. However, we consider the safe exploration setting, in which we require each state-action tuple $(s,a)$  to be safe. Let task policy $\pi_{task}$ interact with the environments, each state action tuple will be checked by the world model to consider safety in the near-future steps through planning. We define the task transitions as $\mathcal{T}_{task}^{\pi}=(s_t, a_t^{\pi_{task}},s_{t+1},r_t)$, where $(s,a) \in \mathcal{S}\times\mathcal{A}:C_H^{\pi_{task}} \leq \boldsymbol{\varepsilon}_{\mathrm{safe}}$ and the recovery transitions $\mathcal{T}_{rec}^{\pi}=(s_t, a_t^{\pi_{rec}},s_{t+1},r_t)$, where $(s,a) \in \mathcal{S}\times\mathcal{A}:C_H^{\pi_{task}} \geq \boldsymbol{\varepsilon}_{\mathrm{safe}}$. In other words, if the action sampled from the task policy $\pi_{task}$ under state $s$ does not satisfy $C_H^{\pi_{task}} \leq \boldsymbol{\varepsilon}_{\mathrm{safe}}$, we would revise the task transition to the recovery transition via the action $a_{rec}$ sampled from $\pi_{rec}$. The workflow is also described in the \textit{safe exploration} area of the online phase in Figure \ref{fig:teaser}. The safe transitions could be defined as:
\begin{equation}
    \mathcal{T}_{\mathrm{safe}}^{\pi}=
\begin{cases}
\mathcal{T}_{\mathrm{task}}^{\pi}, & C_H^{\pi_{task}} \leq \boldsymbol{\varepsilon}_{\mathrm{safe}}\\
\mathcal{T}_{\mathrm{rec}}^{\pi}, & C_H^{\pi_{task}} > \boldsymbol{\varepsilon}_{\mathrm{safe}}
\end{cases}
\end{equation}
Based on the safe transitions, we fine-tune the task policy using Objective \ref{eq:ppo loss} with GAE advantage estimation.

\section{Theoretical Analysis}

We provide theoretical justification for \ours's superior performance through an "action correction" mechanism.

\subsection{Preliminaries}

Let $A^{\pi_{task}}(s, a)$ denote the GAE-based advantage function under policy $\pi_{task}$. We classify states based on H-step constraint violation:
$$\mathcal{S}_{rec} = \{s \in \mathcal{S} \mid \exists a \sim \pi_{task}(\cdot|s): C_H^{\pi_{task}}(s, a) > \varepsilon_{safe}\}$$

We classify actions based on H-step constraint violation at the action level:
$$\mathcal{A}_{risk}(s) = \{a \in \mathcal{A} \mid C_H^{\pi_{task}}(s, a) > \varepsilon_{safe}\}$$
$$\mathcal{A}_{safe}(s) = \{a \in \mathcal{A} \mid C_H^{\pi_{task}}(s, a) \leq \varepsilon_{safe}\}$$
The risk probability at state $s$ is defined as:
$$p_{risk}(s) = \mathbb{P}_{a \sim \pi_{task}(\cdot|s)}[a \in \mathcal{A}_{risk}(s)]$$

\subsection{Assumptions}

\begin{assumption}[Non-trivial Risk Distribution]
\label{assum:risk}
There exists a non-negligible fraction of states where the task policy samples risky actions:
$$\mathbb{E}_{s \sim \rho_{\pi_{task}}}[p_{risk}(s)] > 0$$
where $\rho_{\pi_{task}}$ is the state visitation distribution under policy $\pi_{task}$.
\end{assumption}

\begin{assumption}[Probabilistic Safe Recovery]
\label{assum:recovery}
For any state $s$, the recovery policy provides safe actions with high probability:
$$\mathbb{P}_{a \sim \pi_{rec}(\cdot|s)}[a \in \mathcal{A}_{safe}(s)] \geq 1-\epsilon_{rec}$$
where $\epsilon_{rec} > 0$ is the recovery failure rate.
\end{assumption}

\begin{assumption}[Safe Action Advantage]
\label{assum:advantage}
For states where both safe and risky actions exist, safe actions provide better expected advantage:
$$\mathbb{E}_{a \in \mathcal{A}_{safe}(s)}[A^{\pi_{task}}(s, a)] \geq \mathbb{E}_{a \in \mathcal{A}_{risk}(s)}[A^{\pi_{task}}(s, a)] + \delta$$
where $\delta > 0$ represents the advantage gap between safe and risky actions.
\end{assumption}

\subsection{Main Result}
\begin{theorem}[Action Correction Benefit]
\label{thm:main}
Under Assumptions~\ref{assum:risk}-\ref{assum:advantage}, the policy improvement from \ours's corrected transitions satisfies:

\begin{equation}
\begin{split}
    \Delta J_{\ours} \geq & \Delta J_{baseline} \\
    & + \mathbb{E}_{s \sim \rho_{\pi_{task}}}[p_{risk}(s)] \cdot \delta \cdot (1-\epsilon_{rec}) - O(\epsilon_{rec})
\end{split}
\end{equation}
where $p_{risk}(s) = \mathbb{P}_{a \sim \pi_{task}}[a \in \mathcal{A}_{risk}(s)]$ is the probability of sampling risky actions at state $s$.
\end{theorem}

\begin{proof}[Proof Sketch]
The baseline improvement is $\Delta J_{baseline} = \mathbb{E}_{s,a \sim \pi_{task}}[A^{\pi_{task}}(s,a)]$. 

\ours's action correction yields expected advantage:
\begin{align}
\mathbb{E}_{\ours}[A^{\pi_{task}}(s,a)] &= (1-p_{risk}(s)) \mathbb{E}_{a \in \mathcal{A}_{safe}(s)}[A^{\pi_{task}}(s,a)] \nonumber \\
&\quad + p_{risk}(s) \mathbb{E}_{a \sim \pi_{rec}}[A^{\pi_{task}}(s,a)]
\end{align}

The improvement per state is:
\begin{align}
&\mathbb{E}_{\ours}[A^{\pi_{task}}(s,a)] - \mathbb{E}_{\pi_{task}}[A^{\pi_{task}}(s,a)] \nonumber \\
&= p_{risk}(s) \left(\mathbb{E}_{a \sim \pi_{rec}}[A^{\pi_{task}}(s,a)] - \mathbb{E}_{a \in \mathcal{A}_{risk}(s)}[A^{\pi_{task}}(s,a)]\right)
\end{align}

Under Assumptions 2-3, this difference is at least $p_{risk}(s) \cdot \delta \cdot (1-\epsilon_{rec}) - O(\epsilon_{rec})$. Averaging over states yields the result.
\end{proof}

\subsection{Discussion}
The improvement bound $\mathbb{E}_{s \sim \rho_{\pi_{task}}}[p_{risk}(s)] \cdot \delta \cdot (1-\epsilon_{rec}) - O(\epsilon_{rec})$ shows \ours gains most when: (1) risky states are frequent ($\mathbb{E}_s[p_{risk}(s)]$ large), (2) safe actions significantly outperform risky ones ($\delta$ large), and (3) recovery demonstrations are high-quality ($\epsilon_{rec}$ small). 

This explains \ours's dual benefit: improved safety and enhanced performance by selectively replacing risky actions with recovery alternatives when necessary.

\section{Experiments}

\subsection{FailureBench: Simulation Benchmark for Evaluating Failure-Aware RL}

To evaluate failure-aware RL algorithms, we introduce FailureBench, a benchmark suite comprising modified versions of the MetaWorld environment~\cite{yu2019meta}. FailureBench is designed to incorporate realistic failure scenarios that typically require human intervention in real-world settings. We categorize these failure scenarios into four representative tasks:

\begin{itemize}
    \item \textbf{Sawyer Bounded Push}. The robot must push a puck to the target while keeping it within a bounded workspace. If it pushes the object beyond the boundary, it is considered an \failures[single], simulating scenarios where objects become unreachable or fall to the ground, necessitating human intervention.
    \item \textbf{Sawyer Bounded Soccer}. The robot must hit a ball into a goal while keeping it within a boundary. The dynamic nature of the rolling ball makes it more prone to \failures.
    \item \textbf{Sawyer Fragile Push Wall}. The robot must push an fragile object to a target position behind a wall. If the object collides with the wall, it is considered an \failures[single], simulating real-world scenarios where fragile objects would be damaged upon collision and require replacement.
    \item \textbf{Sawyer Obstructed Push}. The robot must navigate around a fragile vase while pushing the object to its goal location. Any collision with the vase is treated as an \failures[single].
\end{itemize}

\begin{figure}[t]
\vspace{0.5em}
    \centering
    \begin{subfigure}[b]{0.115\textwidth}
        \centering
        \includegraphics[width=\textwidth]{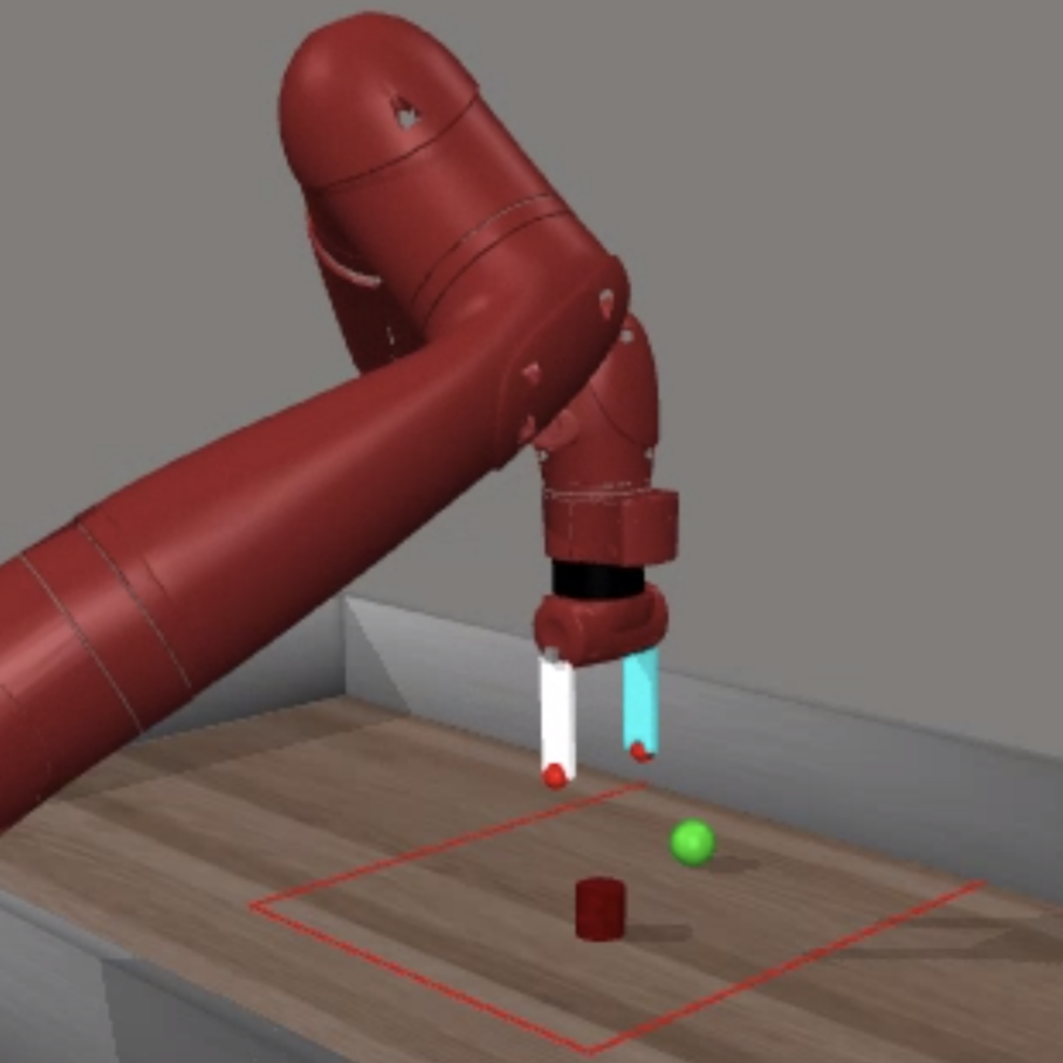}
        \label{fig:task1}
    \end{subfigure}
    \hfill
    \begin{subfigure}[b]{0.115\textwidth}
        \centering
        \includegraphics[width=\textwidth]{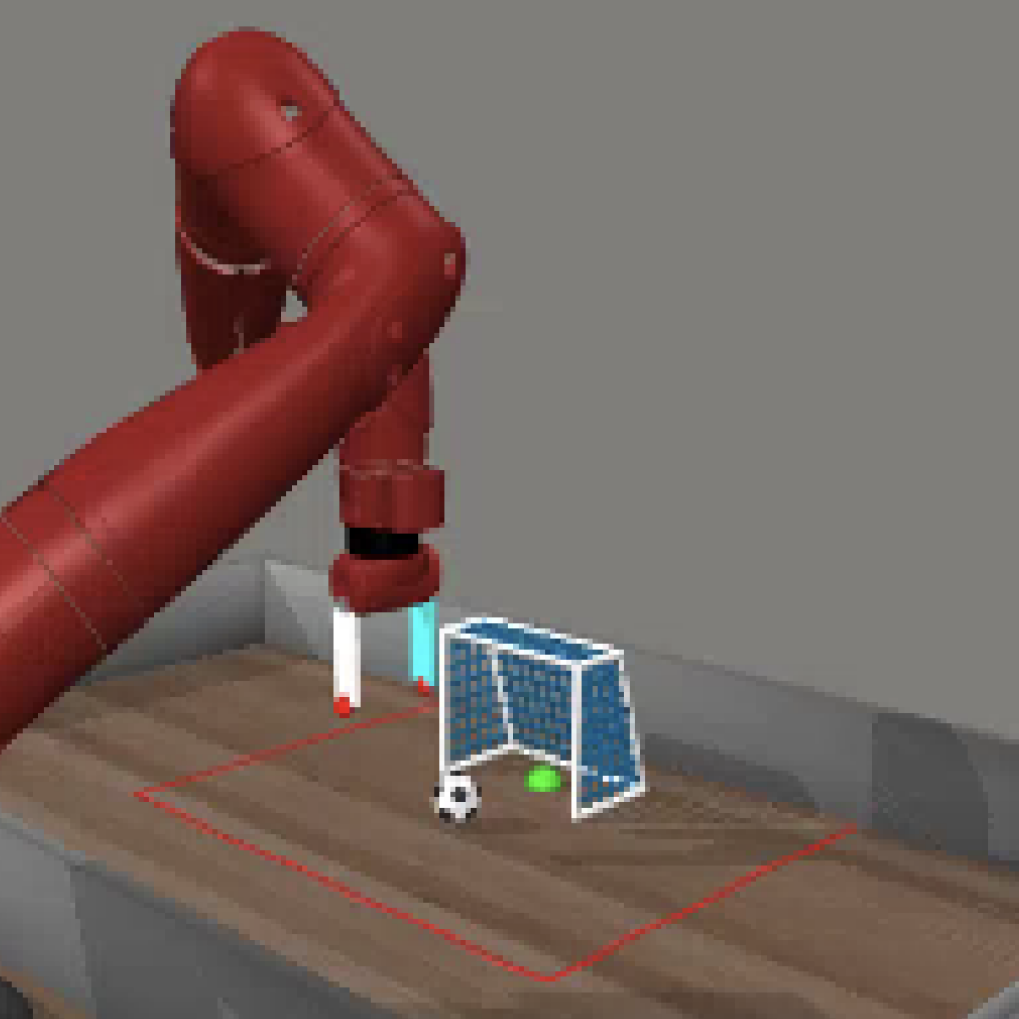}
        \label{fig:task2}
    \end{subfigure}
    \hfill
    \begin{subfigure}[b]{0.115\textwidth}
        \centering
        \includegraphics[width=\textwidth]{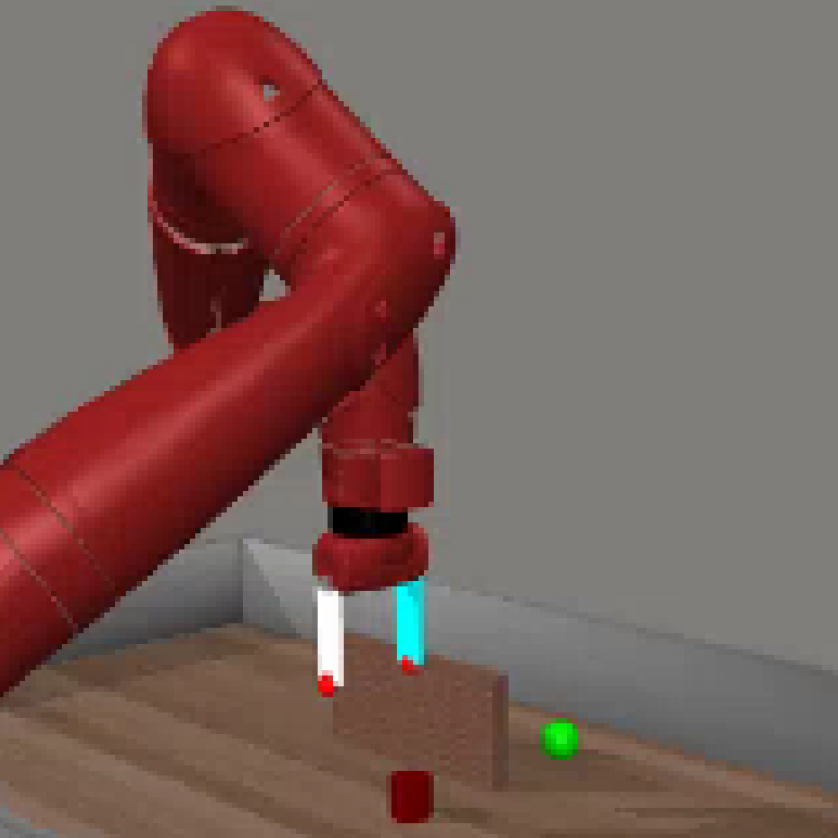}
        \label{fig:task3}
    \end{subfigure}
    \hfill
    \begin{subfigure}[b]{0.115\textwidth}
        \centering
        \includegraphics[width=\textwidth]{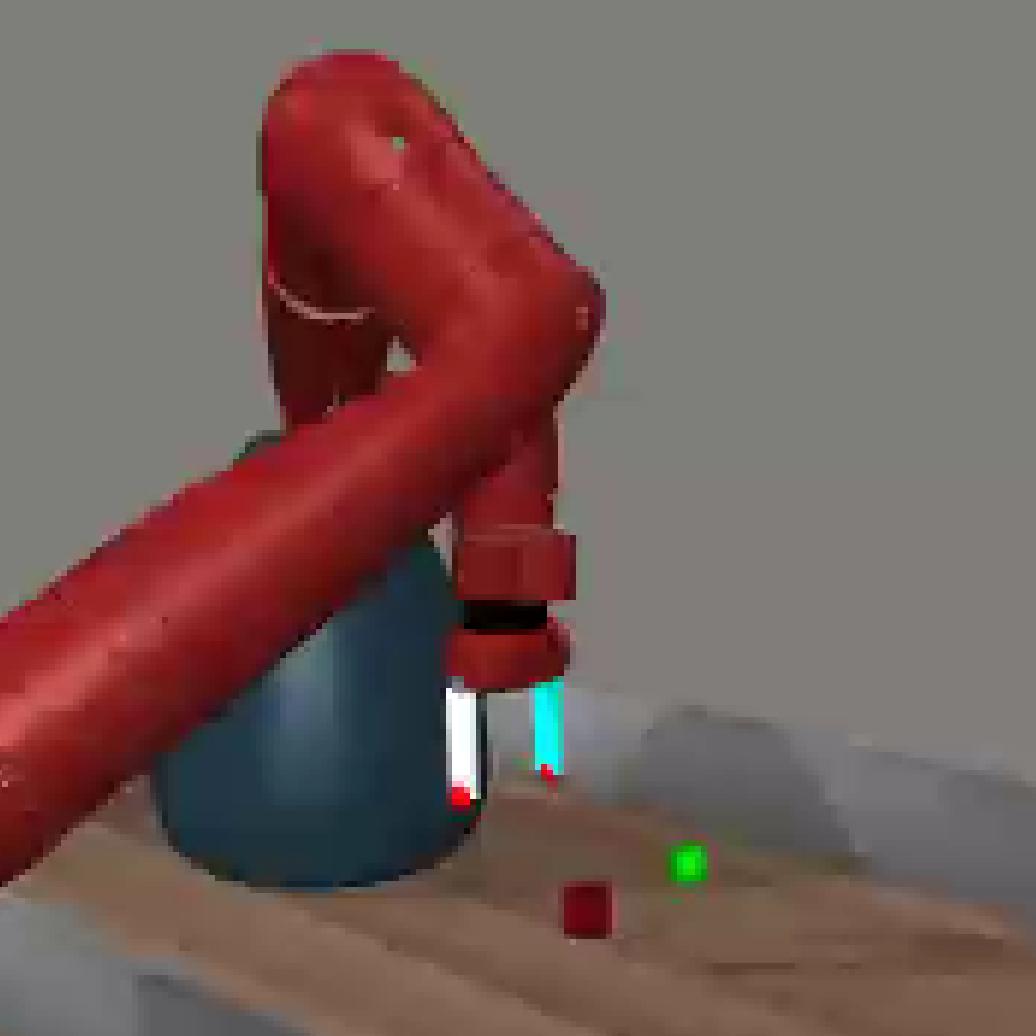}
        \label{fig:task4}
    \end{subfigure}
    \vspace{-1em}
    \caption{Illustration of the four tasks in our FailureBench. From left to right: Bounded Push, Bounded Soccer, Fragile Push Wall, and Obstructed Push.}
    \label{fig:tasks}
    
    \vspace{-1.5em}
    
\end{figure}

\subsection{Simulation Experiments}
We conduct extensive experiments using our proposed FailureBench to evaluate the effectiveness of \ours. We compare \ours against the baseline Uni-O4 algorithm~\cite{lei2024unio}, which represents the state-of-the-art offline-to-online RL approach but without explicit failure avoidance mechanisms.
Additionally, we compare our method with three state-of-the-art online safe RL algorithms, i.e., PPO-Lagrangian~\cite{ray2019benchmarking}, P3O~\cite{p3o}, and CPO~\cite{cpo}, where each is used to fine-tune the same offline pre-trained policy.

\subsubsection{Experimental Setup}

For each environment in FailureBench, we collect three types of demonstration data for offline pre-training:
\begin{itemize}
    \item \textbf{Task demonstrations}: 20 medium expert trajectories for each task collected using BAC~\cite{ji2024seizingserendipityexploitingvalue}, representing sub-optimal demonstration for manipulation tasks.
    \item \textbf{Recovery demonstrations}: 120 trajectories demonstrating recovery behaviors from near-failure states collected with script policy.
    \item \textbf{Failure demonstrations}: 
    20k–200k transition sequences containing failures for world model training. Specifically, we simulate online exploration by adding controlled stochastic noise during the deployment of the offline pre-trained policy. We then collect the failing trajectories, which match the distribution of failures likely to occur during online fine-tuning.

\end{itemize}

We evaluate performance using two key metrics:
\begin{itemize}
    \item \textbf{Average Return}: The mean episodic reward improvement after $10^6$ steps of online finetuning.
    \item \textbf{Failure Episodes}: The number of episodes containing \failures during $10^6$ steps of online finetuning.
\end{itemize}

\subsubsection{Results and Analysis}

Figure \ref{fig:failure_comparison} compares failure episodes between our method \ours and the baseline Uni-O4 across all four environments. Our approach consistently encounters significantly fewer failure episodes with an average reduction of \textbf{43.6\%} across all tasks and up to \textbf{65.8\%} in the most challenging environments. The larger improvements in highly dynamic environments (Bounded Soccer) and complex constraint scenarios (Obstructed Push) highlight our method's strength in handling challenging interaction dynamics.

\begin{figure}[!htbp]
    \centering
    \includegraphics[width=0.9\columnwidth]{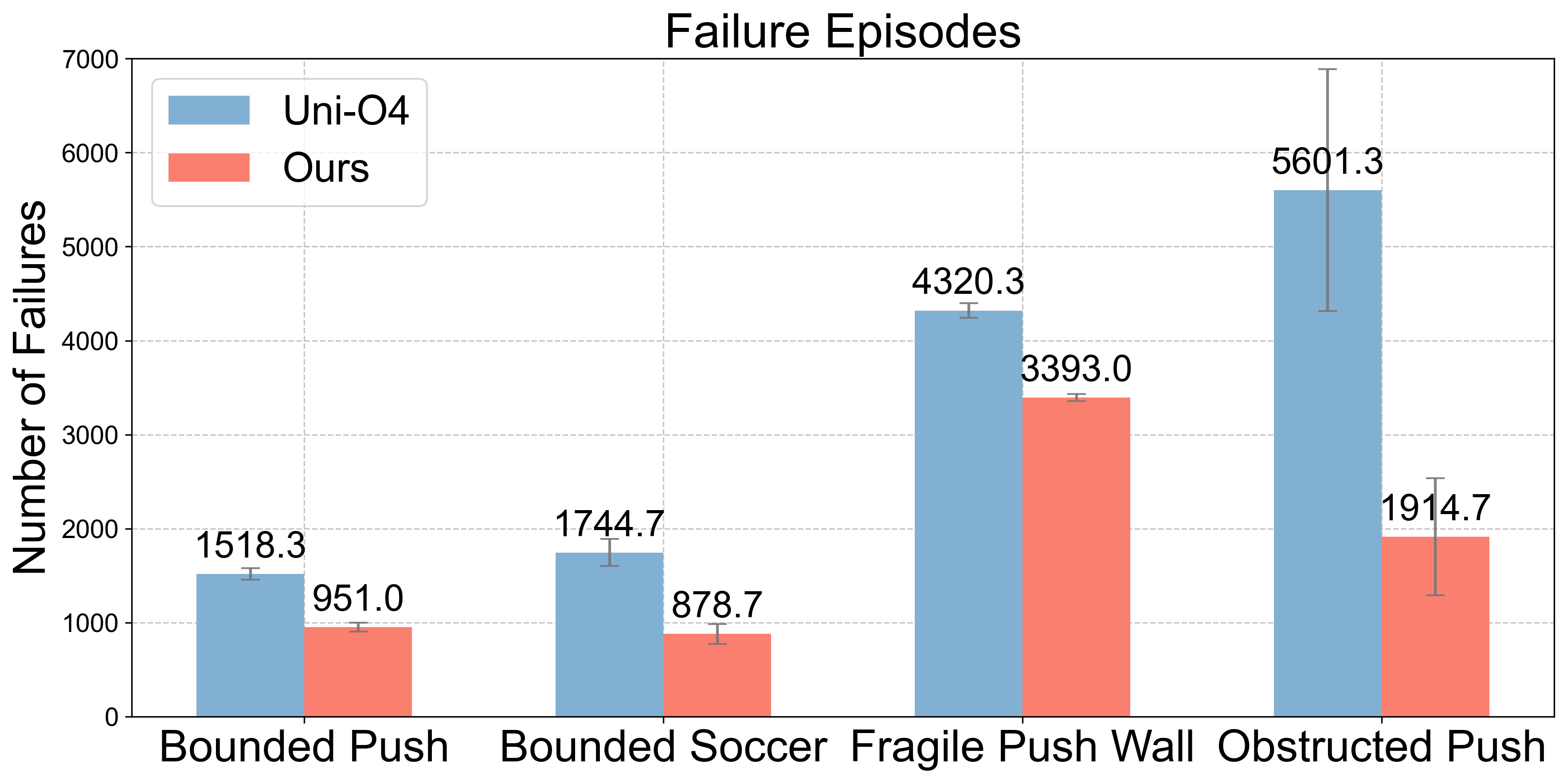}
    \caption{Comparison of average failure episodes during fine-tuning for Uni-O4 (blue) and our method (red) in FailureBench.}
    \label{fig:failure_comparison}
\end{figure}

\begin{figure}[!htbp]

    \centering
    \includegraphics[width=0.9\columnwidth]{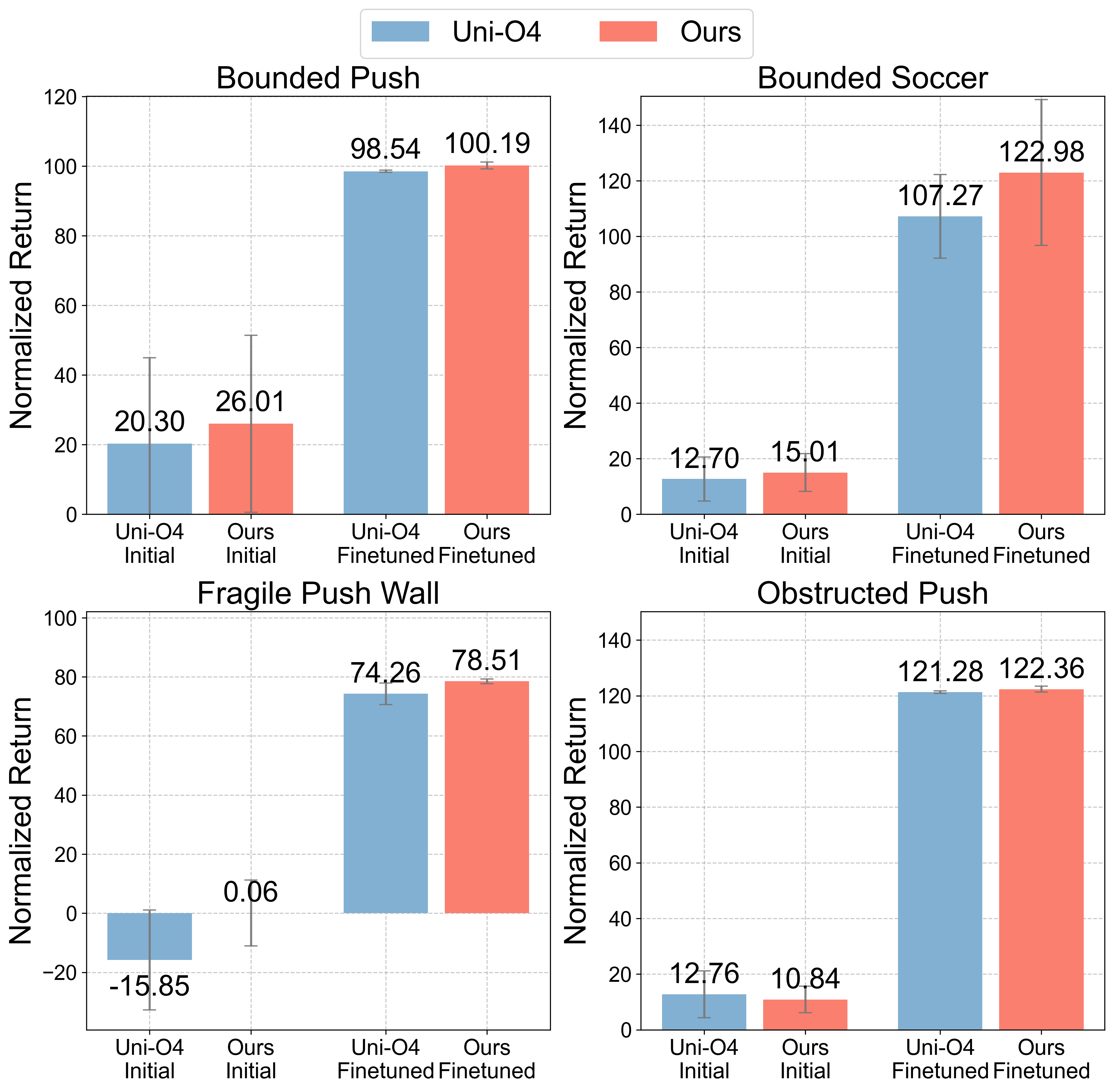}
    \caption{Performance comparison between Uni-O4 (blue) and our method (red) in FailureBench. The bars show average return before and after fine-tuning. All returns are normalized relative to an expert script policy's performance (100).}
    \label{fig:performance_comparison}

\vspace{-1em}

\end{figure}

Analysis of average returns (Figure \ref{fig:performance_comparison}) reveals that despite the additional constraints imposed by our failure-avoiding mechanisms, our approach achieves competitive or higher performance in terms of task rewards. This demonstrates that our method successfully balances exploration and safety, achieving effective learning without sacrificing performance.

Traditional safe RL methods significantly underperform when applied to offline-to-online scenarios, despite careful hyperparameter tuning. Table \ref{tab:safe_rl_comparison} shows that \ours outperforms these approaches by an average of over 800\% across tasks, suggesting a fundamental incompatibility between standard safe RL algorithms and pre-trained policy initialization. This performance gap likely stems from their optimization objectives creating distributional shifts that prevent effective utilization of pre-trained knowledge.

\begin{table}[!htbp]
\centering
\caption{Performance comparison after fine-tuning from the same offline pre-trained policy. Values represent average returns across three seeds.}
\label{tab:safe_rl_comparison}
\resizebox{\columnwidth}{!}{
\begin{tabular}{lcccc}
\toprule
\textbf{Method} & \textbf{Bounded Push} & \textbf{Bounded Soccer} & \textbf{Fragile Push Wall} & \textbf{Obstructed Push} \\
\midrule
\ours (Ours) & \textbf{4593.96 $\pm$ 50.69} & \textbf{2276.53 $\pm$ 554.32} & \textbf{3220.12 $\pm$ 66.67} & \textbf{1227.18 $\pm$ 12.40} \\
PPO-Lag & 420.79 $\pm$ 841.07 & 404.80 $\pm$ 334.67 & 268.83 $\pm$ 695.50 & 543.95 $\pm$ 387.59 \\
P3O & 95.33 $\pm$ 61.51 & 598.55 $\pm$ 653.84 & -278.62 $\pm$ 530.71 & 455.86 $\pm$ 348.82 \\
CPO & 156.28 $\pm$ 212.07 & 692.90 $\pm$ 267.52 & -994.55 $\pm$ 1139.89 & 47.40 $\pm$ 37.26 \\
\bottomrule
\end{tabular}
}
\end{table}

\subsubsection{Ablation Studies}
To better understand the contribution of each component in \ours, we conduct two key ablation studies focused on our failure prediction and avoidance mechanisms:

\noindent \textbf{Study 1. World Model vs. Recovery-RL Safety Critic}

First, we investigate whether our world-model-based safety critic with future rollout prediction is necessary by replacing it with a Q-function safety critic from Recovery-RL~\cite{thananjeyan2021recovery}, which is based on a simple MLP.
This approach estimates the discounted future probability of constraint violation under the current policy:
\begin{align}
    \label{eq:q-safe-def}
    &Q_{\text{safe}}(s_t, a_t) = \mathbb{E}_\pi\left[\sum_{t'=t}^\infty \gamma_{\text{safe}}^{t'-t} c_{t'}|s_t, a_t\right]\\
    &= c_t + (1 - c_t)\gamma_{\text{safe}}\mathbb{E}_\pi\left[Q_{\text{safe}}(s_{t+1}, a_{t+1})|s_t, a_t\right].  
\end{align}

\noindent \textbf{Study 2. Recovery Policy vs. MPPI Planning}

Second, we evaluate the necessity of a separate pre-trained recovery policy by replacing it with Model Predictive Path Integral (MPPI)~\cite{williams2015modelpredictivepathintegral}, a sampling-based model predictive control method. Our implementation plans actions by:

\begin{itemize}
    \item Sampling multiple action trajectories from a Gaussian distribution.
    \item Computing rewards and constraints for each trajectory using the world model.
    \item Filtering trajectories based on a safety threshold ($\varepsilon_{\text{safe}}$).
    \item Selecting the best trajectory among safe options according to expected returns.
\end{itemize}

\noindent \textbf{Result Analysis}

Table \ref{tab:ablation_failure_comparison} presents the failure episode results of our ablation studies. The performance across different variants reveals several interesting insights:

Replacing our world-model-based safety critic with the simple critic from Recovery-RL~\cite{thananjeyan2021recovery} leads to a substantial increase in failure episodes in all environments, which is the most pronounced in Bounded Soccer with 92\% more failures, where complex dynamics make accurate long-term risk assessment crucial. This confirms that our approach's explicit modeling of future state-action sequences provides superior prediction of potential failures compared to a simple Q-function estimator.

MPPI planning underperforms our learned recovery policy, particularly in environments with complex dynamics. This suggests that planning-based methods that implicitly infer recovery action from the world model encounter difficulties in highly dynamic environments. Our recovery policy leverages expert demonstrations of recovery behaviors that would be difficult to discover through planning alone, providing a significant advantage in handling complex failure scenarios.

These results validate our design choices and demonstrate that our combined approach of a world-model-based safety critic and a pre-trained recovery policy provides the most robust performance across diverse failure scenarios.

The return comparison in Figure \ref{fig:ablation_return_comparison} further illustrates that our \ours consistently achieves superior final performance after fine-tuning.

\begin{figure}[!htbp]
    \begin{minipage}{\columnwidth}
        \centering
        \captionof{table}{Comparison of Average Failure Episodes across FailureBench environments.}
        \label{tab:ablation_failure_comparison}
        \resizebox{\columnwidth}{!}{%
        \begin{tabular}{lccc}
        \toprule
        \textbf{Failure Episodes $\downarrow$} & \textbf{\ours (Ours)} & \textbf{Rec-RL} & \textbf{MPPI} \\
        \midrule
        Bounded Push & $\mathbf{951.00 \pm 52.37}$ & $1250.67 \pm 293.33$ & $1112.00 \pm 180.01$ \\
        Bounded Soccer & $\mathbf{878.67 \pm 120.65}$ & $1687.00 \pm 130.85$ & $2022.67 \pm 279.74$ \\
        Fragile Push Wall & $\mathbf{3393.00 \pm 42.51}$ & $4234.00 \pm 56.93$ & $3753.00 \pm 13.23$ \\
        Obstructed Push & $\mathbf{1914.67 \pm 711.46}$ & $2175.00 \pm 128.42$ & $2229.00 \pm 727.12$ \\
        \bottomrule
        \end{tabular}%
        }
    \end{minipage}

    \vspace{1em}
    
    \begin{minipage}{\columnwidth}
        \centering
        \includegraphics[width=0.9\columnwidth]{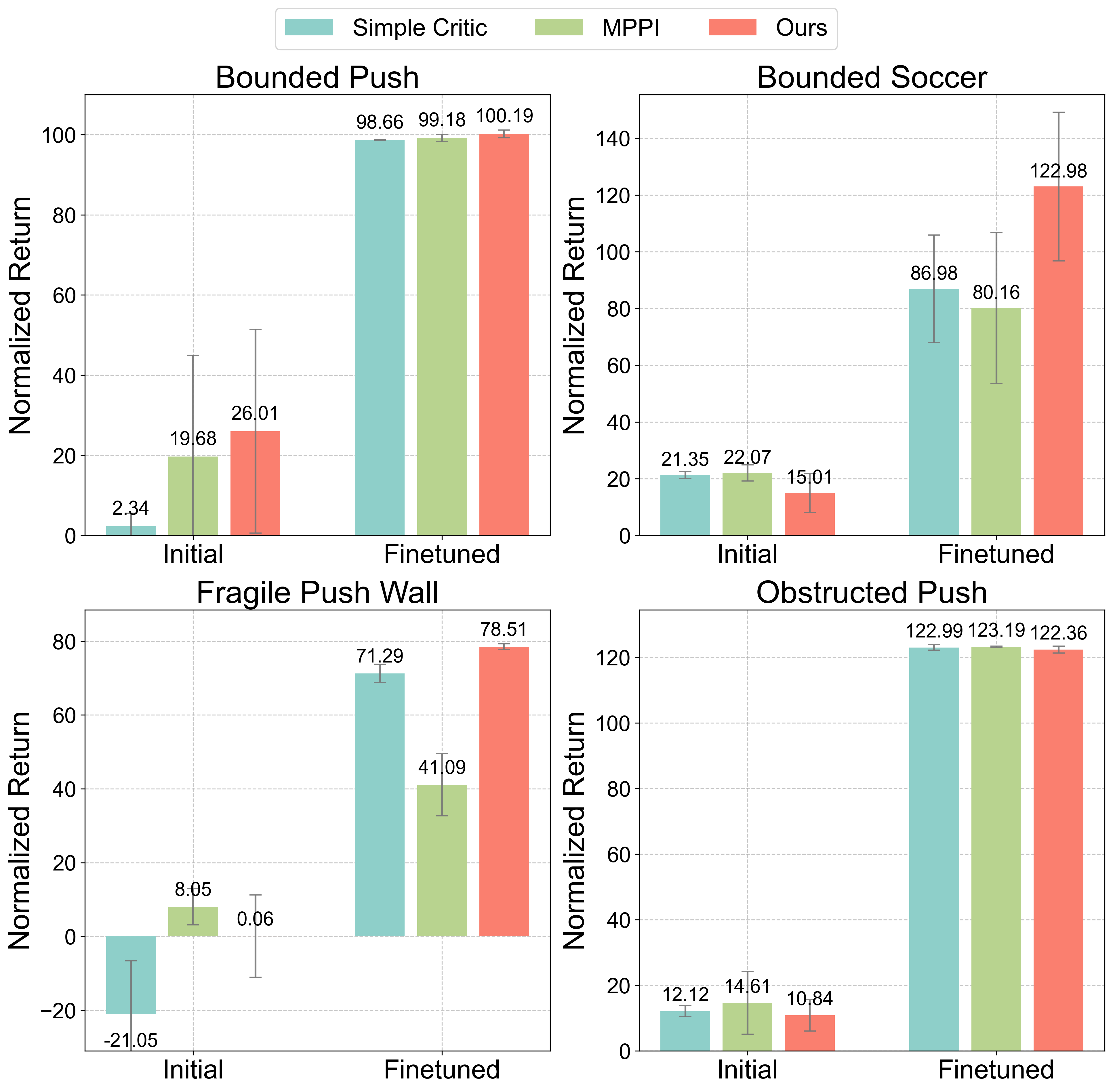}
        \caption{Comparison of average returns before and after fine-tuning for our method and the two ablation variants. All returns are normalized relative to an expert script policy's performance (100).}
        \label{fig:ablation_return_comparison}
    \end{minipage}
\end{figure}

\begin{table}[!htbp]
\centering
\caption{Average return comparison for three Franka tasks.}
\scriptsize
\label{tab:combined_returns}
\begin{tabular}{llcc}
\toprule
\textbf{Task} & \textbf{Method} & \textbf{Initial} & \textbf{Finetuned} \\
\midrule
\multirow{2}{*}{\begin{tabular}[c]{@{}l@{}}Franka Fragile \\ Push Wall\end{tabular}} & Uni-O4 & 343.28 $\pm$ 121.23 & \textbf{369.15 $\pm$ 8.93} \\
 & \ours (Ours) & 349.95 $\pm$ 123.12 & \textbf{363.65 $\pm$ 14.38} \\
\midrule
\multirow{2}{*}{\begin{tabular}[c]{@{}l@{}}Franka Disturbed \\ Push\end{tabular}} & Uni-O4 & 262.62 $\pm$ 129.59 & \textbf{308.29 $\pm$ 79.83} \\
 & \ours (Ours) & 321.49 $\pm$ 149.57 & \textbf{384.52 $\pm$ 87.45} \\
\midrule
\multirow{2}{*}{\begin{tabular}[c]{@{}l@{}}Franka Bounded \\ Soccer\end{tabular}} & Uni-O4 & 512.79 $\pm$ 143.34 & \textbf{615.97 $\pm$ 103.20} \\
 & \ours (Ours) & 578.77 $\pm$ 133.80 & \textbf{638.96 $\pm$ 91.59} \\
\bottomrule
\end{tabular}
\end{table}

\subsection{Real-World Experiments}

To validate the practical effectiveness of \ours, we deploy \ours on a Franka Emika Panda robot to evaluate its performance in challenging real-world scenarios. We implement three representative tasks and assess both failure reduction and performance improvement (Figure \ref{fig:mainteaser}).

\subsubsection{Experimental Setup}

\begin{itemize}
    \item \textbf{Franka Fragile Push Wall.} The robot must push an object, which is assumed to be fragile, to a target position behind a wall. If the object collides with the wall during manipulation, it is considered damaged and requires replacement, indicating an \failures[single] that necessitates human intervention.
    
    \item \textbf{Franka Disturbed Push.} The robot must push an object to a target while avoiding a dynamic obstacle (a decorative flower) that is randomly moved by a human operator, simulating unpredictable environmental changes.

    \item \textbf{Franka Bounded Soccer.} The robot must use a UMI gripper~\cite{chi2024universal} to "kick" a ball toward a target on an uneven turf surface. If the ball rolls beyond the defined boundaries due to the irregular surface dynamics, it is considered an \failures[single].
    
\end{itemize}

We use an Intel RealSense D435 camera for real-time image acquisition. State estimation is performed using YOLOv8 object detection and color-based filtering to track the positions of objects in the workspace. The system operates at approximately 5Hz for both perception and control.

For each task, we conduct training sessions consisting of 50 episodes and compared \ours against the baseline algorithm. We collect 40-80 trajectories for three types of demonstration via teleoperation. Each demonstration type requires approximately 10-20 minutes of teleoperation, resulting in a total human effort of 30-60 minutes per task for collecting all three demonstration types. For the failure demonstrations, specifically, we guide the robot via teleoperation into various failure states that the robot would likely encounter during autonomous execution, providing representative data for safety critic pre-training.

\subsubsection{Results and Analysis}

\begin{figure}[t]
    \centering
    \includegraphics[width=0.9\columnwidth, trim=0 5 0 -5, clip]{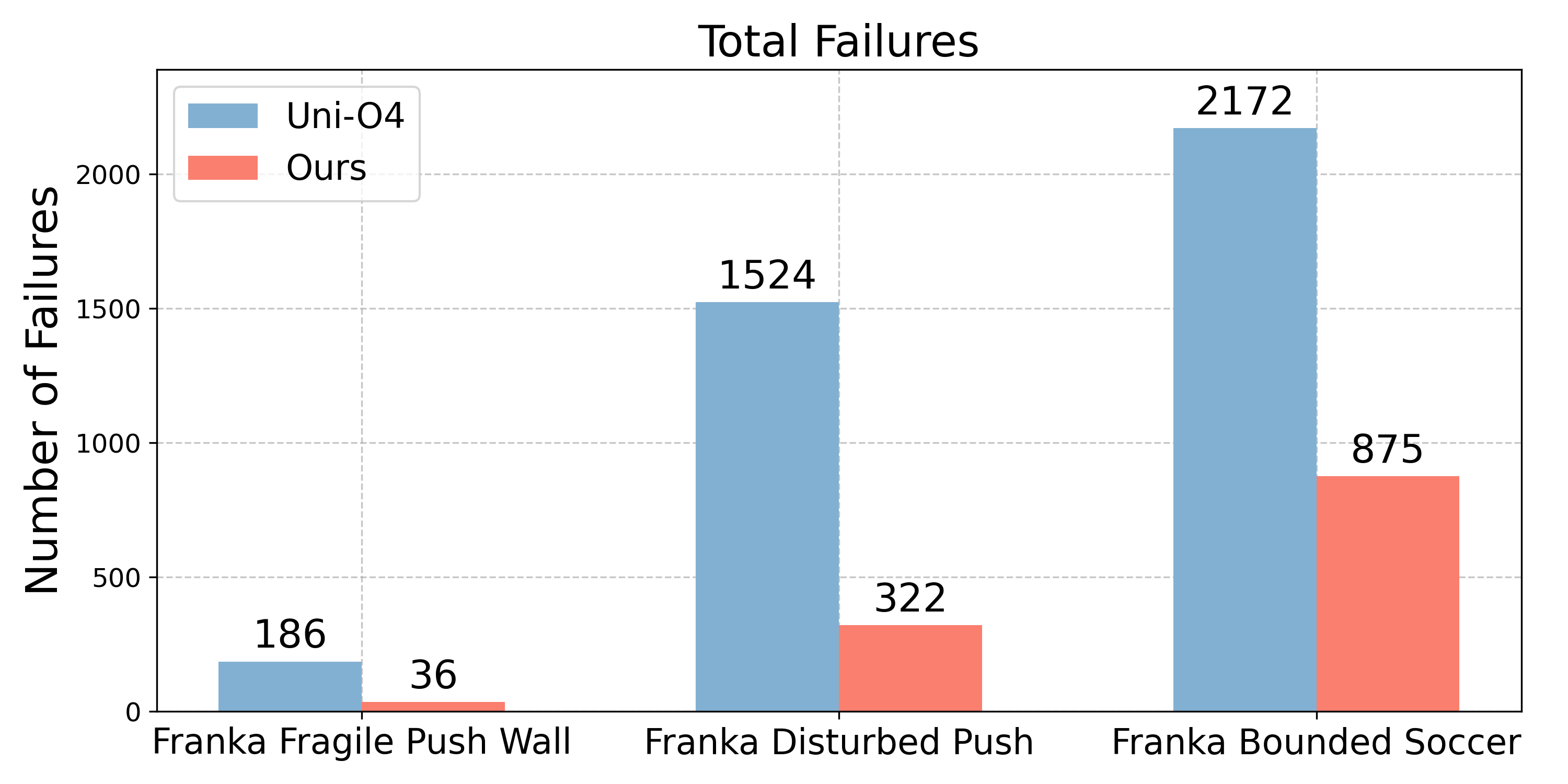}
    \caption{Comparison of total failure state-action pairs during 50 training episodes in real-world experiments.}
    \label{fig:real_failures}

\end{figure}

Figure \ref{fig:real_failures} presents a comparison of total failures during the 50-episode training phase for all environments. The results demonstrate a dramatic reduction in failures when using \ours compared to the baseline.

Table \ref{tab:combined_returns} shows the average return achieved by each method before and after fine-tuning, which further highlights the necessity and superiority of online RL finetuning. Online RL finetuning not only boosts task performance but also significantly decreases the standard deviation, indicating more robust policies. In the Disturbed Push and Bounded Soccer tasks, the finetuned policy achieves noticeably higher returns. This result signifies that \ours successfully leverages the privilege of online RL finetuning while minimizing failures requiring extensive human intervention in the real world.
\section{conclusion and future work}
We introduce a failure-aware offline-to-online reinforcement learning framework that integrates a world-model-based safety critic and a recovery policy to reduce \failures during exploration significantly. 

\textbf{Limitations and future work.} Our current work does not incorporate additional modalities such as 2D/3D vision, tactile sensing, and others, which could further enrich the perceptual capabilities of the system. In future work, we plan to extend \ours to address mobile manipulation and dual-arm robotic platforms and to integrate these other modalities. 

Another promising direction is to pre-train large-scale failure world models across diverse tasks and embodiments, potentially enabling better generalization to unseen failure modes in new domains.

\section*{Acknowledgements}
We would like to thank Zhecheng Yuan, Tianming Wei, Guangqi Jiang, and Xiyao Wang for their valuable discussions.
This work was supported by the Tsinghua University Dushi Program.

\printbibliography
\begin{acronym}
\acro{HP}{high-pass}
\acro{LP}{low-pass}
\end{acronym}

\end{document}